%% file: main.tex
\documentclass[sigconf, screen, dvipsnames, nonacm]{acmart}

\usepackage{hyperref}
\usepackage[noend]{algorithmic}
\usepackage{algorithm}
\usepackage{multirow}
\usepackage{amsmath}
\usepackage{bbm}
\usepackage{booktabs}
\usepackage[inline]{enumitem} 
\usepackage{subcaption}
\usepackage{bm} 
\usepackage{xcolor}
\usepackage{balance}
\usepackage{lipsum}

\usepackage{wrapfig}

\usepackage{arydshln}
\makeatletter
\def\adl@drawiv#1#2#3{
        \hskip.5\tabcolsep
        \xleaders#3{#2.5\@tempdimb #1{1}#2.5\@tempdimb}%
                #2\z@ plus1fil minus1fil\relax
        \hskip.5\tabcolsep}
\newcommand{\cdashlinelr}[1]{%
  \noalign{\vskip\aboverulesep
          \global\let\@dashdrawstore\adl@draw
          \global\let\ adl@draw\adl@drawiv}
  \cdashline{#1}
  \noalign{\global\let\adl@draw\@dashdrawstore
          \vskip\belowrulesep}}
\makeatother

\usepackage{tikz}
\usetikzlibrary{automata, arrows, bayesnet, bending}
\usepackage{mathtools}

\usepackage{tikzscale}

\newcommand{\indep}{\perp \mkern-9.5mu \perp}

\copyrightyear{2026}
\acmYear{2026}
\setcopyright{rightsretained}

\begin{document}
\title{Accelerating A/B-Tests with Counterfactual Estimation}
\subtitle{Reducing Variance through Policy Overlap}

\author{Olivier Jeunen}\authornote{Part of this work was done while the author was affiliated with aampe.}
\affiliation{
  \institution{Independent Researcher}
  \city{Antwerp}
  \country{Belgium}
}

\begin{abstract}
Online controlled experiments are the gold standard for hypothesis testing in online platforms.
Notwithstanding their ubiquity, they are notoriously expensive to run, and issues of variance hamper statistical power in assessing treatment effects.
While standard variance reduction techniques leverage model-based control variates to reduce outcome noise, they remain agnostic to potential structural relationships between competing policies.

In this work, we identify a critical inefficiency in the standard A/B-testing protocol: when a treatment and control policy agree on an action, the resulting outcome contributes noise but no signal regarding the treatment effect---unnecessarily inflating confidence intervals.
We propose a novel experimental protocol that exploits this \emph{policy overlap} to accelerate experimentation. 
The key insight is to frame the randomised treatment assignment mechanism as a meta-policy, and leverage $\Delta$-Off-Policy Estimation methods to obtain unbiased estimates for average treatment effects.
We prove analytically that our approach recovers standard A/B-testing practices in the general case, but that its variance scales with the \emph{divergence} between policies rather than raw outcome variance.
Hence, we dominate the standard Difference-in-Means estimator whenever
policies have common support, and the improvement is strict whenever the
overlap region contributes non-zero residual variance.
Empirical results corroborate these theoretical insights---holding promise for significant impact on the real-world evaluation of recommender systems, information retrieval pipelines, and large language model interfaces.
\end{abstract}

\maketitle

\input{1.Introduction}
\input{2.Background}
\input{3.Methodology}
\input{4.Experiments}
\input{5.Conclusions}

\balance

\bibliographystyle{ACM-Reference-Format}
\bibliography{bibliography}
\begin{table*}[h!]
    \centering
    \caption{Methodological Comparison of Contemporaneous Policy-Aware A/B-Testing Estimators.}
    \label{tab:contemporaneous_comparison}
    \resizebox{\textwidth}{!}{%
    \begin{tabular}{llll}
        \toprule
         & \textbf{This Work} & \textbf{Konishi et al. (MID) \cite{konishi2026accuratealgorithmcomparisonab}} & \textbf{Sakhi et al. \cite{Sakhi2025}} \\
        \midrule
        \textbf{Core Mechanism} & Meta-policy mixture $\pi_0 = p\pi + (1-p)\pi'$ & Synthetic middle policy $\pi_M$ & $f$-regularized IPS weights \vspace{.8ex}\\
        
        \textbf{Theoretical Basis} & Exact point-wise variance dominance & Reward variance bounding via $\pi_{M^*}$ & Second-moment surrogate $S_f$ \vspace{.8ex}\\
        \textbf{Traffic Allocation} & Analytical variance-optimal $p^*$ & Assumed 50/50 balanced split & Sample ratio $n_r$ integrated \vspace{.8ex}\\
        \textbf{Reward Modeling} & $\Delta$-DR and $\Delta$-MRDR objective & IPS; DR extensions noted & IPS-focused; no reward model \vspace{.8ex}\\
        \textbf{Combinatorial Spaces} & $\Delta$-DCG for Ranking & Single-action bandits & Single-action; non-Markovian \vspace{.8ex}\\
        \textbf{Robustness} & Assumes known propensities & Assumes known propensities & Extends to propensity estimation \\
        \bottomrule
    \end{tabular}%
    }
\end{table*}
\appendix

\section{Contemporaneous Work}
Table \ref{tab:contemporaneous_comparison} provides a systematic overview of contemporaneous approaches to policy-aware A/B testing, highlighting how these distinct methodologies---whilst sharing a common foundation in counterfactual inference---diverge in their mathematical frameworks, variance-reduction guarantees, and practical extensions.

\end{document}

%% file: 1.Introduction.tex
\section{Introduction \& Motivation}
A/B-tests make the internet go round: online controlled experiments are used far and wide as the go-to approach for testing and evaluating virtually any change to an online platform or application~\cite{kohavi2020trustworthy}.
They are, however, not infallible, and the potential pitfalls with these digital randomised controlled trials have been discussed at length in the research literature~\cite{Kohavi2022,Dmitriev2017,Jeunen2023_Forum,Kohavi2024,Deng2016,Jeunen2025_ttest}.

A commonly recurring problem is that of statistical power: often, the variance inherent to the outcome metric inhibits practitioners from attaining tight confidence intervals around Average Treatment Effects (ATEs) that would allow them to claim statistical significance and have confidence that any observed effects are beyond the result of sampling variation.
User-level outcomes such as click-through rate, dwell time, watch time, or revenue proxies are inherently noisy, heavy-tailed, and subject to substantial heterogeneity.
When expected improvements are small---as is typical for incremental model updates---detecting lift requires either prolonged test durations or substantial traffic allocation, both of which impose material opportunity costs.
Variance reduction in online experimentation is, as a result, a prominent research area with significant industry involvement.
Existing approaches tend to leverage model-based control variates to explain outcome heterogeneity~\cite{Deng2013, Poyarkov2016, Baweja2024,Budylin2018,Guo2021}, or ensure that the metrics they consider have inherently favourable noise characteristics~\cite{Xie2016, Kharitonov2017,Richardson2023,Tripuraneni2023,Jeunen2024_Learning,Deng2024,Li2020_cupac}.

Alternatively, a large body of work has explored ``Offline A/B-Testing''~\cite{Gilotte2018} in an attempt to forego the opportunity cost associated with online experiments by leveraging data logged under randomised policies.
These methods make use of counterfactual inference techniques, often based on importance sampling or Inverse Propensity Scoring (IPS)~\cite[Ch.~9]{Owen2013}, to construct offline estimators of online effects~\cite{Jeunen2021Thesis}.
This has led to widespread successes in both offline evaluation and learning capabilities in Recommender Systems~\cite{Vasile2020,Saito2021,Jeunen2023_nDCG,Gruson2019,chen2019top}, Information Retrieval~\cite{Gupta2024_WSDM,Joachims2017,Oosterhuis2020}, Large Language Models (LLMs)~\cite{Shao2024,Zheng2025}, and beyond~\cite{Bottou2013,Sagtani2024,vandenAkker2024}.

Some recent work has drawn connections between the on- and offline paradigms, opening up a promising research area~\cite{Jeunen2026_ABisOPE,Zhang2025}.

In particular, \citet{Jeunen2026_ABisOPE} demonstrates a structural and exact equivalence between the classical Difference-in-Means (DiM) estimator for A/B-testing, and the $\Delta\beta^{\star}$-IPS estimator used in off-policy settings~\cite{Jeunen2024_DeltaOPE}.
Indeed, the randomised treatment assignment mechanism  that diverts traffic to either $\pi_A$ or $\pi_B$ is in itself a meta-policy $\pi_0$.
This simple reframing allows us to directly use the $\Delta$-OPE framework on A/B-testing data---logged under $\pi_0$---to unbiasedly estimate the ATE between $\pi_A$ and $\pi_B$ and, furthermore, enjoy a \emph{guaranteed} variance reduction if $\pi_A$ and $\pi_B$ have \emph{any} common support.

The central observation of this work is simple but consequential: when two policies agree on an action, the observed outcome contributes noise but no information about their ATE.
If, for a given context, both policies would select the same action, then the realised reward is identical under treatment and control.
Such samples are uninformative for estimating lift, yet the standard DiM estimator assigns them equal weight.
The result is that common support dilutes statistical power in A/B-testing---whereas $\Delta$-OPE leverages it to appropriately downweight the relevant samples.

In practice, most updates to production models are incremental.
As a consequence, the treatment and control policies often agree on a large fraction of actions.
For many contexts, they select identical items or near-identical slates.
This high degree of policy overlap is common, yet its statistical implications are rarely exploited.

Our key contributions include:
\begin{enumerate}
    \item \textbf{Applying the $\Delta$-OPE framework to online A/B-testing.}
    We show that by interpreting treatment assignment as a meta-policy, we can directly apply $\Delta$-OPE estimators to A/B-testing data.
    This unlocks a family of unbiased estimators for the ATE that leverage information about policy overlap.
    
    \item \textbf{A variance dominance theorem.}  
    We formally prove that the resulting estimators strictly dominate the DiM estimator for any traffic split provided the policies have nonzero overlap, and show how the variance scales with the divergence between the policies that are being tested.

    \item \textbf{An optimal traffic allocation insight.}  
    By characterising the variance of the ATE as a function of the treatment allocation ratio, we show that the variance-optimal split need not be balanced. The optimal allocation is determined by the curvature of the policy divergence and admits a unique solution, which can be used to guide practitioners.

    \item \textbf{The $\Delta$-MRDR estimator to directly minimise variance.}
    Leveraging connections between $\Delta$-Doubly Robust (DR) and online variance reduction methods, we propose $\Delta$-MRDR to train reward models that directly minimise ATE estimation variance, concentrating model capacity on regions of policy disagreement where variance reduction is most impactful.

    \item \textbf{The $\Delta$-DCG estimator to enable ranking applications.}
    Extending recent work that frames the classical Discounted Cumulative Gain (DCG) metric as an off-policy estimator~\cite{Jeunen2023_nDCG} with the $\Delta$-OPE view, we derive a $\Delta$-DCG estimator that unbiasedly estimates the ATE for ranking policies under the Position-Based Model.
    We additionally incorporate the $\beta$-IPS control variate~\cite{Gupta2024} to further reduce variance.

    \item \textbf{Empirical validation and a practical deployment recipe.} 
    Through rigorous, realistic and reproducible simulation scenarios, we demonstrate that substantial variance reduction can be achieved with a minimal engineering overhead. 
\end{enumerate}

%% file: 2.Background.tex
\section{Background \& Problem Setting}
In the context of A/B-testing~\cite{kohavi2020trustworthy}, we are particularly interested in the setting where the hypothesis tested by the online experiment pertains to two varying personalised treatment regimes: policies $\pi,\pi^{\prime}$~\cite{Vasile2020}.
Policies induce probability distributions over actions $a \in \mathcal{A}$, conditional on contextual information $x \in \mathcal{X}$, as $\pi(a|x) \equiv \mathsf{P}(A=a|X=x;\Pi=\pi)$.
The ``action'' framing is general, but it subsumes the common use-cases where $A$ represents item recommendations~\cite{Joachims2021}, rankings~\cite{Gupta2024_WSDM}, sequences of tokens (LLMs)~\cite{Ouyang2022}, or model parameters themselves~\cite{Jeunen2024_MOO}.
The personalisation aspect comes from conditioning on the context $x$, which can include any information pertaining to the request and end user.

Online experiments often consider user-level metrics~\cite{Jeunen2024_RecSysIndustry} (e.g. conversions, revenue, clicks).
The ATE of a policy deployment on this metric is then used as the decision criterion in a test, with:
\begin{gather*}    
V(\pi) = \mathop{\mathbb{E}}\limits_{x \sim \mathsf{P}(X)}\mathop{\mathbb{E}}\limits_{a \sim \pi(\cdot|x)}\mathop{\mathbb{E}}\limits_{\mathsf{P}(Y|X=x;A=a)}\left[Y\right], \\
V_{\Delta}(\pi,\pi^{\prime}) = V(\pi) - V(\pi^{\prime}).
\end{gather*}

\subsection{Online Controlled Experiments}
When running an A/B-test, both policies are effectively deployed to a fraction of the user population.
This enables unbiased estimation of the ATE by separating the sample according to the policy variant that was assigned to a user or context, as:
\[
V_{\Delta}(\pi,\pi^{\prime}) = \underbrace{\mathop{\mathbb{E}}\limits_{a\sim\pi(\cdot|x)}[Y]}_{V(\pi)} - \underbrace{\mathop{\mathbb{E}}\limits_{a\sim\pi^{\prime}(\cdot|x)}[Y]}_{V(\pi^{\prime})}.
\]

For a dataset $\mathcal{D} \coloneqq \{(x_i, a_i, y_i, \pi_i)\}_{i=1}^{N}$,  we denote a subset generated under policy $\pi$ as  $\mathcal{D}_\pi \coloneqq \{(x_i, a_i, y_i, \pi_i) \in \mathcal{D}|\pi_i=\pi\}$.
Then, the difference in sample means gives rise to the standard ATE estimate:
\begin{gather}
\hat\mu(Y,\pi) = \frac{1}{|\mathcal{D_{\pi}}|} \sum_{(x_i,a_i,y_i) \in \mathcal{D}_{\pi}} y_i,\\
\hat V_{\Delta-{\rm DiM}}(\pi, \pi^{\prime}) =\hat \mu(Y,\pi)-\hat \mu(Y,\pi^{\prime}).
\label{eq:dim}
\end{gather}
The inherent variance of the outcomes $y_i$ can hamper the statistical power of this estimator.
A common approach is to leverage a model-based additive control variate to reduce this variance.
CUPED~\cite{Deng2013}, CUPAC~\cite{Li2020_cupac}, and MLRATE~\cite{Guo2021} can be unified as regression-adjusted difference-in-means estimators (RADiM)~\cite{Jeunen2026_ABisOPE}:
\begin{gather}
\hat\mu_f(Y,\pi) = \frac{1}{|\mathcal{D_{\pi}}|} \sum_{(x_i,a_i,y_i) \in \mathcal{D}_{\pi}} \left( y_i - f(x_i) \right), \\
    \hat V_{\Delta-{\rm RADiM}}(\pi, \pi^{\prime}) =\hat \mu_f(Y,\pi)-\hat \mu_f(Y,\pi^{\prime}).
\end{gather}

\subsection{Off-Policy Estimation}
To remedy the high costs associated with online evaluation, ideas from the broader causal and counterfactual inference literature~\cite{Bottou2013,Saito2021} have been applied to offline evaluation~\cite{Gilotte2018}.
Data logged under some logging policy $\pi_0$ can be used to unbiasedly estimate the ATE, typically by leveraging importance sampling techniques~\cite[Ch.~9]{Owen2013}.
For notational convenience, let $\Delta(a|x) \coloneqq \pi(a|x) - \pi^{\prime}(a|x)$.
The $\Delta$-IPS estimator is given by~\cite{Jeunen2024_DeltaOPE}:
\begin{equation}\label{eq:delta_IPS}
    \hat V_{\Delta\rm-IPS} = \frac{1}{|\mathcal{D}|}\sum_{(x,a,y) \in \mathcal{D}} \frac{\Delta(a|x) }{\pi_0(a|x)}y.
\end{equation}
Whilst unbiased, the importance weights can be problematic for variance.
Analogous to the online experiment setting, additive control variates provide a way to reduce variance whilst preserving unbiasedness.
With a single scalar, this is known as $\beta$-IPS~\cite{Jeunen2024_DeltaOPE,Gupta2024}:
\begin{equation}\label{eq:beta_ips}
    \hat V_{\Delta\beta\rm-IPS} = \frac{1}{|\mathcal{D}|}\sum_{(x,a,y) \in \mathcal{D}} \frac{\Delta(a|x) }{\pi_0(a|x)}(y-\beta).
\end{equation}
The optimal, variance-minimising, value for $\beta$ is given by:
\begin{equation}
\beta^{\star} = \frac{\mathbb{E}\left[\left(\frac{\Delta(a|x) }{\pi_0(a|x)}\right)^2 y\right]}{\mathbb{E}\left[\left(\frac{\Delta(a|x) }{\pi_0(a|x)}\right)^2 \right]}.
\end{equation}
This quantity can be estimated from logged data through a cross-fitting procedure to preserve unbiasedness of the estimator~\cite{Jeunen2026_AdditiveDominates,Chernozhukov2018}.

More generally, the additive control variate can be derived from a learnt reward model $f(x,a)$, yielding the $\Delta$-Doubly Robust (DR) estimator~\cite{Dudik2014,Jeunen2026_ABisOPE}.
We define the control variate $\hat{\tau}(x)$ as the expected difference in model predictions under the target policies:
\begin{equation}
\hat{\tau}(x) \coloneqq \mathop{\mathbb{E}}\limits_{a \sim \pi}[f(x,a)] - \mathop{\mathbb{E}}\limits_{a \sim \pi'}[f(x,a)] = \sum_{a \in \mathcal{A}} \Delta(a|x)f(x,a).
\end{equation}
With a correction term to preserve unbiasedness, we obtain:
\begin{equation}\label{eq:delta_dr}
\hat V_{\Delta\text{-DR}} = \frac{1}{|\mathcal{D}|} \sum_{(x,a,y) \in \mathcal{D}} \left( \hat{\tau}(x) + \frac{\Delta(a|x)}{\pi_0(a|x)} (y - f(x,a)) \right).
\end{equation}

Note that if $f$ is independent of actions $a$, $\hat \tau(x)\equiv0$, as is the case for typical online applications of RADiM methods~\cite{Jeunen2026_ABisOPE}.

The variance reduction that doubly robust methods entail, relies heavily on the reward model $f$~\cite{Jeunen2020_DR}.
\citet{Farajtabar2018} propose the More Robust Doubly Robust (MRDR) objective to directly learn a parameterised model $f_{\theta}$ that minimises estimation variance.

\citet{Jeunen2026_ABisOPE} formally derives an exact equivalence between $\Delta\beta^{\star}$-IPS and DiM, and between $\Delta$-DR and RADiM.
This implies that extensions to any of these methods can be applied interchangeably, opening the door for cross-pollination among disconnected research areas.
Nevertheless, they do not consider that the structural overlap between the policies further enhances statistical power.

A growing body of work has sought to connect online and offline evaluation methods.
\citet{Oosterhuis2020_Logopt} propose an approach that blends interleaving with counterfactual estimation to efficiently evaluate ranking policies online.
\citet{Buchholz2024} introduce a dynamic user behaviour model that optimises its bias-variance trade-off to reduce estimation error in offline evaluation.

A practical instantiation that blends A/B-testing with OPE was recently detailed by \citet{Zhang2025} in the context of marketplace ranking---proposing a heuristic framework that leverages rank disagreement to reduce the variance of proxy metrics in online experiments.
Concurrent work by \citet{konishi2026accuratealgorithmcomparisonab} introduces the MID estimator, which bounds reward variance via a synthetic middle policy for single-action bandits, though it relies on an assumed balanced traffic split.
Similarly, recent work by \citet{Sakhi2025} explores a family of bias-corrected regularised estimators for online experiments, specifically accommodating non-Markovian reward processes.
While they also empirically demonstrate the value of exploiting policy similarity, their analysis relies on minimising a variance surrogate rather than the exact variance, and treats the traffic allocation as fixed.
Furthermore, both works explicitly leave the integration of doubly robust control variates to future work---overlooking the exact equivalences between DR and RADiM estimators.
Appendix A provides a comparison table.


Our work complements and extends this literature by identifying how $\Delta\beta^\star$-IPS can be directly applied to online experiments with policies, to provide a guaranteed variance reduction as soon as they have common support.
We derive the exact variance of our estimators and solve for the optimal traffic allocation ratio $p^\star$ that minimises it.
This naturally enables us to apply $\Delta$-DR for online variance reduction, and we further propose the $\Delta$-MRDR learning objective that focuses model capacity on regions of policy disagreement to maximise statistical power.
Finally, we naturally extend our work to include ranking estimators for DCG~\cite{Jeunen2023_nDCG}.

%% file: 3.Methodology.tex
\section{Methodology \& Contributions}
Exact equivalences between DiM and IPS, and RADiM and DR have been established in the literature~\cite{Jeunen2026_ABisOPE}---but only at the level where policies themselves are treated as actions.
That is, they show that standard A/B-testing practices imply an action space of treatment policies $T \in \{\pi, \pi'\}$ where the logging policy selects $\pi_0(T=\pi)=p$.
The policies themselves are modelled as black boxes, and the resulting importance weights solely rely on the A/B-group assignment:
\begin{equation}
    w_{\mathrm{DiM}}(x,a) = 
    \begin{cases} 
        \frac{\bm{1}\{T=\pi\} -\bm{1}\{T=\pi^{\prime}\} }{\pi_0(T=\pi)} =\frac{1}{p} & \text{if } a \sim \pi, \\
        \frac{\bm{1}\{T=\pi\} -\bm{1}\{T=\pi^{\prime}\} }{\pi_0(T=\pi^{\prime})}   =\frac{-1}{1-p} & \text{if } a \sim \pi'.
    \end{cases}
\end{equation}
Plugging these importance weights into Eq.~\ref{eq:beta_ips} with an optimal $\beta^{\star}$, recovers the standard and widely used DiM estimator~\cite[\S 3]{Jeunen2026_ABisOPE}.

A key insight is that we can jointly and directly model the action distributions induced by the policies $\{\pi,\pi^{\prime}\}$, by considering $\pi_0$ as a meta-policy that defines a mixture over the policies being tested:
\begin{equation}
    \pi_0(a|x) = p \pi(a|x) + (1-p) \pi'(a|x).
\end{equation}

Whilst mathematically and conceptually straightforward, this gives rise to a different set of importance weights that can be plugged into Eq.~\ref{eq:beta_ips} to obtain an unbiased ATE estimate:
\begin{equation}\label{eq:policy_aware_props}
    w_{\mathrm{\Delta}}(x,a) = \frac{\Delta(a|x)}{\pi_0(a|x)}.
\end{equation}
The resulting policy-aware estimator leverages structural information about the policies to reduce estimation noise for the ATE.
Crucially, because of the exact equivalence between DiM and $\Delta\beta^{\star}$-IPS, this comes at no extra cost. 
Even if $\pi$ and $\pi$ operate on fully disjoint action spaces, jointly modelling them under the $\Delta$-OPE framework recovers standard practice.
As soon as there is any common support, i.e. $\exists a \in \mathcal{A}: \pi(a|x)>0 \land \pi^{\prime}(a|x)>0$, and this overlapping region contributes non-zero residual variance, the policy-aware estimator strictly reduces estimation variance compared to standard practice---and these results naturally extend to their regression-adjusted analogues RADiM and DR.

\subsection{Variance Reduction and Dominance}
We now prove that the policy-aware estimator strictly dominates the standard approach in terms of variance---and hence, estimation error and statistical power---under minimal assumptions.

\begin{theorem}[Variance Reduction with $w_\Delta$]
Let $\hat{V}_{\mathrm{DiM}}$ be the standard Difference-in-Means estimator and $\hat{V}_{\Delta\beta^{\star}\mathrm{-IPS}}$ be the policy-aware estimator leveraging the importance weights in Eq.~\ref{eq:policy_aware_props}.
For any context $x$, action space $\mathcal{A}$, and treatment allocation ratio $p \in (0, 1)$:
\begin{equation}
    \mathrm{Var}(\hat{V}_{\mathrm{\Delta\beta^{\star}-IPS}})
    \le
    \mathrm{Var}(\hat{V}_{\mathrm{DiM}}).
\end{equation}
The inequality is strict whenever there exists a set of positive probability on which
$$
\pi(a|x)>0,\quad
\pi'(a|x)>0,\quad\text{and}\quad
\mathbb{E}\!\left[(Y-\beta^\star)^2\mid x,a\right]>0.
$$
\end{theorem}

\begin{proof}
Both estimators are unbiased for the same Average Treatment Effect. Note that $\hat{V}_{\mathrm{DiM}}$ is baseline-invariant: plugging $w_{\mathrm{DiM}}$ into Eq.~\ref{eq:beta_ips} with any $\beta$ recovers Eq.~\eqref{eq:dim} exactly, as the $\beta$-terms cancel under fixed allocation $p$.
We therefore compare $\hat{V}_{\mathrm{DiM}}$ against $\hat{V}_{\Delta\beta^\star\text{-IPS}}$: the variance-optimal baseline correction.

Consequently,
\[
\mathrm{Var}(\hat V)
=
\mathbb{E}[\hat V^2]
-
V_\Delta(\pi,\pi')^2,
\]
and comparing variances is equivalent to comparing second moments.

For a single sample $(x,a)$, the second moment of the estimators evaluated at this shared baseline is
\[
\mathbb{E}\!\left[w^2(Y-\beta^\star)^2\mid x,a\right].
\]
We define $\alpha=\pi(a|x)$ and $\gamma=\pi'(a|x)$ for brevity.

\textbf{1. Variance of the DiM Estimator.}
The squared weight for DiM is stochastic conditioned on the action, depending on the group assignment $T$.
Due to the law of total expectation over assignments, and noting that $\pi_0=p\alpha+(1-p)\gamma$, we have
\begin{equation}
    \mathbb{E}[w_{\mathrm{DiM}}^2\mid x,a]
    =
    \left(\frac1p\right)^2\frac{p\alpha}{\pi_0}
    +
    \left(\frac{-1}{1-p}\right)^2\frac{(1-p)\gamma}{\pi_0}
    =
    \frac1{\pi_0}
    \left(
    \frac{\alpha}{p}
    +
    \frac{\gamma}{1-p}
    \right).
\end{equation}

\textbf{2. Variance of the Policy-Aware Estimator.}
The policy-aware importance weight is deterministic for a given action:
\begin{equation}
    w_\Delta^2
    =
    \frac{(\alpha-\gamma)^2}{\pi_0^2}.
\end{equation}

\textbf{3. The Variance Gap.}
The difference in squared weights is
\begin{equation}
    \Delta w^2
    =
    \mathbb{E}[w_{\mathrm{DiM}}^2\mid x,a]
    -
    w_\Delta^2.
\end{equation}

Factoring out the common denominator $p(1-p)\pi_0^2$ gives
\begin{equation}
    \Delta w^2
    =
    \frac1{\pi_0^2}
    \left[
    \frac{\pi_0}{p(1-p)}
    \big((1-p)\alpha+p\gamma\big)
    -
    (\alpha-\gamma)^2
    \right].
\end{equation}

Substituting $\pi_0=p\alpha+(1-p)\gamma$ yields
\begin{multline}
    \Delta w^2
    =
    \frac1{p(1-p)\pi_0^2}
    \Big[
    \big(p\alpha+(1-p)\gamma\big)
    \big((1-p)\alpha+p\gamma\big)
    \\
    -
    p(1-p)(\alpha-\gamma)^2
    \Big].
\end{multline}

Expanding the terms cancels $\alpha^2$ and $\gamma^2$, leaving
\begin{equation}
    \Delta w^2
    =
    \frac{\alpha\gamma}
    {p(1-p)\pi_0^2}
    \ge 0.
\end{equation}

Therefore,
\[
\Delta w^2\,(Y-\beta^\star)^2
\ge0
\]
pointwise, since both $\Delta w^2\ge0$ and $(Y-\beta^\star)^2\ge0$.

Then, expectations yield:
\[
\mathrm{Var}(\hat{V}_{\mathrm{DiM}})
-
\mathrm{Var}(\hat{V}_{\Delta\beta^\star\mathrm{-IPS}})
=
\mathbb{E}
\!\left[
\Delta w^2
(Y-\beta^\star)^2
\right]
\ge0.
\]

The inequality is strict whenever there exists a set of points on which both $\alpha\gamma>0$ and
\[
\mathbb{E}\!\left[(Y-\beta^\star)^2\mid x,a\right]>0,
\]
since the integrand is then strictly positive.
\end{proof}
Since all estimators we consider are unbiased---a reduction in variance implies a reduction in overall estimation error.
Moreover, lower variance leads to tighter confidence intervals and, as a result, an increase in statistical power.
We construct an analogous proof for their additive control variate analogues RADiM and $\Delta$-DR.

\begin{theorem}[Variance Reduction with $\Delta$-DR]
Let $\hat{V}_{\mathrm{RADiM}}$ be the standard Regression-Adjusted Difference-in-Means estimator and $\hat{V}_{\Delta\mathrm{-DR}}$ be the Policy-Aware Doubly Robust estimator.
For any reward model $f: \mathcal{X} \times \mathcal{A} \to \mathbb{R}$, context $x$, and allocation $p \in (0, 1)$:
\begin{equation}
    \mathrm{Var}(\hat{V}_{\Delta\mathrm{-DR}}) \le \mathrm{Var}(\hat{V}_{\mathrm{RADiM}}).
\end{equation}
The inequality is strict if there exists any overlap between policies, i.e. $\exists a \in \mathcal{A}: \pi(a|x) > 0 \land \pi^{\prime}(a|x) > 0$, and the reward model has non-zero residual variance, i.e. $\mathrm{Var}(Y|x,a) > 0$.
\end{theorem}

\begin{proof}
Both estimators are unbiased and share the exact same expected control variate term $\hat{\tau}(x) = \mathbb{E}_{a\sim\pi}[f(x,a)] - \mathbb{E}_{a\sim\pi'}[f(x,a)]$, which is constant given $x$.
Consequently, the difference in their total variance is fully determined by the variance of their weighted residuals.
Let $Z = Y - f(x,a)$ denote the residual.
The estimators differ only in the weights applied to $Z$: $\hat{V}_{\mathrm{RADiM}}$ uses the assignment-based weight $w_{\mathrm{DiM}}$, while $\hat{V}_{\Delta\mathrm{-DR}}$ uses the policy-aware weight $w_{\Delta}$.

Applying the pointwise weight identity established in Theorem~3.1 to the residual variable $Z$ instead of the direct outcome $Y$ yields:
\begin{equation}
    \Delta\mathrm{Var} = \mathrm{Var}(\hat{V}_{\mathrm{RADiM}}) - \mathrm{Var}(\hat{V}_{\Delta\mathrm{-DR}}) = \mathbb{E}\left[ (w_{\mathrm{DiM}}^2 - w_{\Delta}^2) Z^2 \right].
\end{equation}
Substituting the strictly non-negative weight difference derived in Theorem 1, leads to:
\begin{equation}
    \Delta\mathrm{Var} = \mathbb{E}_{x}\left[ \sum_{a \in \mathcal{A}} \frac{\pi(a|x)\pi'(a|x)}{p(1-p)\pi_0(a|x)^2} \mathbb{E}[Z^2 | x,a] \right].
\end{equation}
Since probabilities, $p(1-p)$, and the expected squared residual $\mathbb{E}[Z^2|x,a]$ are all non-negative, the variance gap is guaranteed to be non-negative.
This implies that $\Delta$-DR can only reduce the variance obtained via RADiM.
\end{proof}

\subsection{Deriving an Optimal Allocation Ratio}
Standard A/B-testing practice recommends a balanced split ($p=0.5$), as it maximises power under the assumption that outcome variances are homoskedastic~\cite{Kohavi2014}.
In contrast, the variance of our policy-aware estimators depends on the structural relationship between policies.
This implies that we can seek the optimal allocation ratio $p \in (0,1)$ that minimises the total estimation variance for a given online experiment.
This, in turn, maximises statistical power and reduces the cost of running the experiment.

For the general case covering both $\Delta\beta^{\star}$-IPS and $\Delta$-DR, we minimise the variance of the weighted variable $Z$ (where $Z=Y$ for IPS, and $Z=Y-f(x,a)$ for DR).
The variance as a function of $p$ is then:
\begin{equation}\label{eq:variance_treatment_alloc}
    J(p) = \mathop{\mathbb{E}}_{x\sim \mathsf{P}(X)} \left[ \sum_{a \in \mathcal{A}} \frac{\Delta(a|x)^2}{p\pi(a|x) + (1-p)\pi'(a|x)} \mathbb{E}[Z^2|x,a] \right].
\end{equation}
Assuming homoskedasticity on $Y$, the term $\mathbb{E}[Z^2|x,a]$ is constant and does not affect the location of the optimum for standard IPS.
In the case of $\Delta$-DR, however, this term represents the residual variance of the reward model. 
Including it ensures that the optimal allocation $p^{\star}$ shifts traffic towards the policy that samples regions where the reward model is least accurate (i.e. has high residual variance), in an attempt to stabilise the overall estimator.

Since $f(y)=1/y$ is strictly convex, the objective $J(p)$ remains strictly convex in $p$ regardless of the weighting by $\mathbb{E}[Z^2|x,a]$.

In practice, assuming a reasonable degree of stationarity in the context distribution $\mathsf{P}(X)$, practitioners can leverage existing system logs to estimate these moments pre-experiment. 
Estimating the residual variance under a production policy yields an approximation for $p^{\star}$ under $\Delta$-DR, ignoring this term yields the optimal $p^{\star}$ under $\Delta\beta^{\star}$-IPS.
We can then apply search or gradient-based optimisation procedures to obtain the variance-optimal ratio a priori.

\paragraph{Adaptive Allocation.}
The above derivation assumes homoskedasticity on $Y$.
Without this assumption, the optimum $p^\star$ relies on a quantity that is unknown prior to the experiment.
To mitigate this, one can consider a practical two-stage adaptive design.
In a first stage, deploy the experiment with a balanced split $p_0=0.5$ to collect an initial sample.
Using this data, empirically estimate the weighted residual variances for both treatment and control groups to solve for the empirical optimum $\hat{p}^\star$.
This lifts the homoskedasticity assumption on $Y$, but induces a stationarity assumption on $\mathsf{P}(X)$.
Note that this procedure does not constitute ``peeking''~\cite{Johari2017}---as the allocation update depends on the empirical variance and not the ATE estimate.
Furthermore, because the $\Delta$-IPS and $\Delta$-DR estimators rely on importance sampling, they remain unbiased for time-varying logging policies, provided that the propensity scores $\pi_0^{(t)}$ used in the estimator accurately reflect the assignment probabilities at each time step $t$ and are treated appropriately~\cite{Oosterhuis2021,Elvira2019, Kallus2021}.
This could allow practitioners to dynamically optimise statistical power without inflating Type-I error rates~\cite{Kohavi2024} in real-world experiments.

\subsection{$\Delta$-MRDR: Optimal Variance Reduction}
Standard regression approaches train the DR reward model $f_\theta$ (parameterised by $\theta$) to minimise the Mean Squared Error (MSE) on the outcome $y$. 
However, minimising model prediction error does not necessarily minimise the variance of the resulting estimator~\cite{Farajtabar2018}.
To address this, we propose the \textbf{$\Delta$-MRDR} (More Robust Doubly Robust) objective, which adapts the model loss specifically to minimise the variance of the $\Delta$-DR estimator defined in Eq.~\ref{eq:delta_dr}.

Since the estimator is unbiased, minimising its variance is equivalent to minimising the second moment of the weighted residual term.
Neglecting the variance of the control variate (which is constant with respect to the action assignment given the context), the optimisation problem becomes:
\begin{equation}
    \theta^{\star} = \mathop{\mathrm{arg\,min}}_{\theta} \mathbb{E}_{(x,a,y) \sim \mathcal{D}} \left[ \left( \frac{\Delta(a|x)}{\pi_0(a|x)} (y - f_\theta(x,a)) \right)^2 \right].
\end{equation}
This yields a Weighted Least Squares (WLS) objective, where the sample weights are given by the squared policy-aware importance weights:
\begin{equation} \label{eq:mrdr}
    \mathcal{L}_{\Delta\text{-MRDR}}(\theta) = \frac{1}{|\mathcal{D}|} \sum_{(x,a,y) \in \mathcal{D}} w_{\Delta}(x,a)^2 (y - f_\theta(x,a))^2.
\end{equation}

Unlike standard regression which treats all samples equally, $\Delta$-MRDR weights the loss by the squared policy divergence $w_{\Delta}^2 \propto (\pi(a|x) - \pi'(a|x))^2$.
Consequently, the objective forces the model to focus its limited capacity on the regions of disagreement between the policies. 
Accurate prediction in regions where policies agree (i.e. $\Delta(a|x) \approx 0$) contributes negligibly to variance reduction and is effectively down-weighted.
This aligns the learning objective with the overarching experimental goal: orienting model capacity towards the context-action-reward triplets that affect the ATE.

\subsection{$\Delta$-DCG: Evaluating Ranking Policies}\label{sec:delta_dcg}
Our work so far has focused on contextual bandit settings where a policy selects a single action.
Modern recommender systems, however, typically serve ranked lists of items to users.
To combat a combinatorial explosion of the action space, models of user behaviour are typically adopted~\cite{Li2018}.

Recent work has shown that Discounted Cumulative Gain (DCG) admits an interpretation as an importance sampling estimator of expected online reward under a position-based click model (PBM) with known exposure probabilities~\cite{Jeunen2023_nDCG}. Under this view, ranking evaluation reduces to off-policy estimation with \emph{exposure} propensities. We leverage this perspective to derive a $\Delta$-DCG estimator for statistically efficient A/B-testing of ranking policies.

Let $\mathcal{R}$ and $\mathcal{R}'$ be two ranking policies we wish to compare.  
Given a context $x$, each policy induces expected exposure for item $a$~\cite{Diaz2020}:
\[
\varepsilon_\mathcal{R}(x,a), \qquad \varepsilon_{\mathcal{R}'}(x,a),
\]
where $\varepsilon_\mathcal{R}(x,a)$ denotes the probability that policy $\mathcal{R}$ exposes item $a$ to the user.
This combines both selection bias due to the stochastic ranking policy, as well as position bias in the system.

As before, the A/B-test logging policy is a mixture:
\[
\varepsilon_0(x,a) = p\,\varepsilon_\mathcal{R}(x,a) + (1-p)\,\varepsilon_{\mathcal{R}'}(x,a).
\]

The common support assumption for IPS is trivially satisfied as before, i.e.
if $\varepsilon_\mathcal{R}(x,a) > 0$ or $\varepsilon_{\mathcal{R}'}(x,a) > 0$, then $\varepsilon_0(x,a) > 0$.

Under the PBM, reward decomposes as:
\[
y = v \cdot q(x,a),
\]
where $v$ is a position-dependent binary examination variable and $q(x,a)$ represents intrinsic ``quality'' or relevance.
In this setting, DCG can be considered an unbiased estimator of an online reward metric.
This requires independence assumptions among ranks and trajectories, for which we refer to the work of \citet{Jeunen2023_nDCG}.

Let $\mathcal{D}= \{(x_i, \bm{a}_i, \bm{y}_i)\}_{i=1}^{N}$ be logged interactions under the mixture policy---where $\bm{a}=(a_1,\ldots,a_K)$ and $\bm{y} = (y_1,\ldots,y_K)$ represent vectorised actions and rewards across positions $1\leq j\leq K$.  
The expected DCG difference between $\mathcal{R}$ and $\mathcal{R}'$ is given by:
\[
\Delta_{\mathrm{DCG}} = \mathop{\mathbb{E}}\limits_{\bm{a} \sim \mathcal{R}} \left[ \sum_{j=1}^{K}y_j \right] - \mathop{\mathbb{E}}\limits_{\bm{a} \sim \mathcal{R}'} \left[ \sum_{j=1}^{K}y_j \right].
\]

By leveraging the exposure propensities, we define the marginal policy-aware importance weight for an item $a$ at position $j$ as:
\[
w_{\Delta}(x, a_j) = \frac{\varepsilon_{\mathcal{R}}(x,a_j) - \varepsilon_{\mathcal{R}'}(x,a_j)}{\varepsilon_0(x,a_j)}.
\]
We can then unbiasedly estimate the difference:
\[
\hat V_{\Delta\mathrm{-DCG}} = \frac{1}{|\mathcal{D}|} \sum_{(x,\bm{a},\bm{y}) \in \mathcal{D}} \sum_{j=1}^{K} w_{\Delta}(x, a_j) y_j.
\]

This estimator is the ranking analogue of $\Delta$-IPS under the PBM, replacing action propensities with exposure propensities that incorporate both selection and position bias. The variance of $\hat V_{\Delta\mathrm{-DCG}}$ depends on the squared marginal importance weights $w_{\Delta}(x, a_j)^2$.

\paragraph{Optimal Baselines for Rankings ($\Delta\beta^\star$-DCG)}
While the marginalisation inherent to the PBM prevents the combinatorial explosion of slate-level importance weights, the estimator still suffers from high variance when policy divergence is large.
To mitigate this, we can introduce an additive control variate directly into the marginal formulation.
Analogous to the optimal scalar baseline $\beta^\star$ defined for the single-action setting, we can derive a variance-minimising baseline for the ranking scenario, as previously proposed for the Item-Position model~\cite{Jeunen2026_AdditiveDominates}.
As the authors note, a globally optimal vectorised additive control variate is non-trivial to estimate.
Instead, we adopt a position-specific approximation that assumes independence among ranks as:
\[
\beta_{\indep,j}^\star = \frac{\mathbb{E} \left[ w_{\Delta}(x,a_j)^2 y_j \right]}{\mathbb{E} \left[ w_{\Delta}(x,a_j)^2 \right]},
\]
yielding the $\Delta\beta_{\indep}^\star$-DCG estimator:
\[
\hat V_{\Delta\beta_{\indep}^\star\mathrm{-DCG}} = \frac{1}{|\mathcal{D}|} \sum_{(x,\bm{a},\bm{y}) \in \mathcal{D}} \sum_{j=1}^{K} w_{\Delta}(x, a_j) (y_j - \beta_{\indep,j}^\star).
\]

By shifting the outcome at each rank based on the global average reward for that specific position, weighted by policy divergence, $\Delta\beta_{\indep}^\star$-DCG substantially reduces variance compared to $\Delta$-DCG. 

When the two ranking policies induce similar exposure probabilities for most items (high \emph{policy overlap} in exposure space), the numerator of $w_\Delta$ is small. Thus, impressions where both rankers behave similarly are effectively down-weighted, mirroring the overlap effect observed in the single-action setting, and ensuring that the estimator focuses strictly on the ATE. Conversely, when policies disagree strongly (large exposure divergence), impressions receive larger weights. In these high-divergence regimes, the position-specific $\beta_{\indep, j}^\star$ baseline is critical to constrain the variance explosion inherently associated with importance sampling.

%% file: 4.Experiments.tex
\begin{figure*}[!t]
    \centering
    \includegraphics[width=\linewidth]{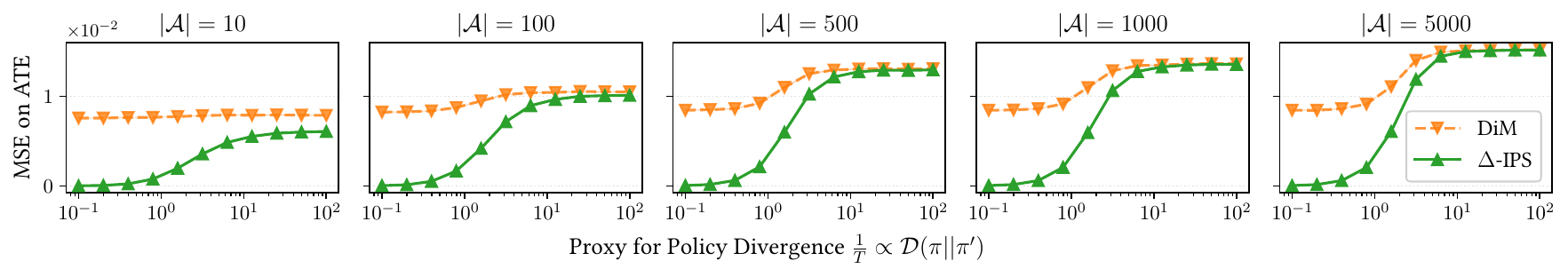}
    \caption{Mean Squared Error (MSE) of the standard DiM estimator and the Policy-Aware $\Delta$-IPS estimator for varying action space sizes $|\mathcal{A}|$. The $x$-axis represents policy divergence (proxied by the inverse temperature $\frac{1}{T}$ of $\pi^{\prime}$). We observe that the $\Delta$-OPE view successfully exploits policy overlap to significantly reduce estimation error, remaining robust across large action spaces.}
    \label{fig:bias_variance}
\end{figure*}

\section{Empirical Validation}
\label{sec:experiments}

We validate our theoretical findings through a series of controlled simulations.
Unlike fixed public datasets where the ground truth counterfactuals are unobserved, synthetic environments allow us to compute the exact bias and variance of our estimators against the true Average Treatment Effect (ATE). Our experiments are designed to answer the following four research questions:

\begin{description}
    \item[\textbf{RQ1 (Baseline Estimator Efficiency)}:]~\\\textit{Does the Policy-Aware $\Delta$-OPE view reduce variance compared to the standard Difference-in-Means (DiM) estimator, and how does this reduction scale with action space size $|\mathcal{A}|$?}
    \item[\textbf{RQ2 (Optimal Experiment Design)}:]~\\\textit{Does the theoretical variance proxy $J(p)$ accurately predict empirical estimation error? Can we achieve lower variance by deviating from the standard even split in asymmetric settings?}
    \item[\textbf{RQ3 (Optimal Variance Reduction)}:] ~\\\textit{Does the proposed $\Delta$-MRDR loss function yield control variates that reduce estimator variance more effectively than $\Delta$-DR?}
    \item[\textbf{RQ4 (Extensions to Ranking)}:]~\\\textit{Do the proposed $\Delta$-DCG estimator and its additive control variate extension reliably reduce estimation variance?}
\end{description}
To aid in the reproducibility of our empirical results, the source code to reproduce all figures and numbers reported in this Section is available at \href{https://github.com/olivierjeunen/delta-OPE-ABTesting/}{github.com/olivierjeunen/delta-OPE-ABTesting/}.

\subsection{RQ1: Estimator Efficiency \& Scalability}
To assess the baseline performance of our estimator, we simulate a contextual bandit setting with a linear reward structure ($d=5$).
We vary the action space size $|\mathcal{A}| \in \{10, 100, 500, 1\,000, 5\,000\}$.
We vary the inverse temperature $\frac{1}{T}$ of a Softmax policy $\pi^{\prime}$, which is a proxy for the policy divergence w.r.t. a uniformly random policy $\pi$, (i.e. $\lim_{\frac{1}{T}\to0} \pi^\prime\equiv\pi$).
We simulate an A/B-testing scenario with $5\,000$ users, and repeat this $1\,000$ times to reduce sampling variation.

\textit{Result.}
Figure~\ref{fig:bias_variance} illustrates the Mean Squared Error (MSE) of the standard DiM estimator versus the proposed $\Delta$-IPS estimator. Consistent with Theorem 3.1, the policy-aware estimator strictly dominates DiM. Notably, in the high-overlap regime ($\frac{1}{T} \to 0$), the variance of $\Delta$-IPS approaches zero as the control variate term perfectly cancels out the common noise. Furthermore, Figure~\ref{fig:bias_variance} demonstrates robustness to the ``curse of dimensionality'': the variance reduction persists even as the action space grows to $|\mathcal{A}|=5\,000$, confirming that $\Delta$-IPS is suitable for larger-scale settings.
Nevertheless, we see strong potential for further variance reduction in larger action spaces by combining the $\Delta$-OPE view with recent advances in the OPE literature~\cite{Saito2022_MIPS,Saito2023}.

\begin{figure}[!t]
    \centering
    \includegraphics[width=\linewidth]{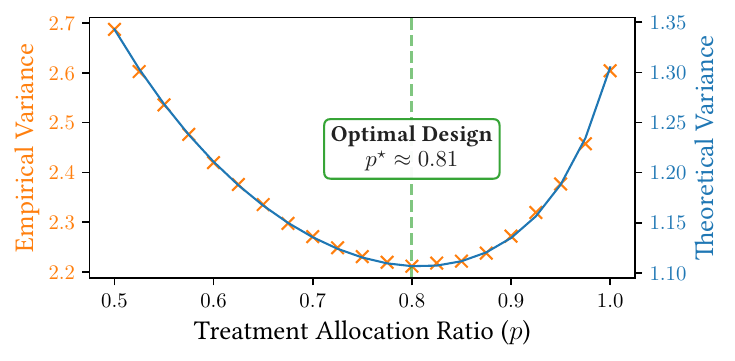}
    \caption{Visualising optimal treatment allocation ($p^\star$) in an asymmetric policy setting. We compare a conservative control policy $\pi'$ against an exploratory treatment policy $\pi$.
    We validate that empirical estimation variance tracks the theoretical variance proxy $J(p)$ derived in Eq.~\ref{eq:variance_treatment_alloc} up to a multiplicative constant. The minimum variance is achieved at $p^\star \approx 0.81$, reducing variance by $18\%$ compared to an even split.}
    \label{fig:optimal_design}
\end{figure}

\subsection{RQ2: Optimal Treatment Allocation}
Standard A/B-testing practice recommends an equal traffic split ($p=0.5$), as this maximises the statistical power of the DiM estimator~\cite{Kohavi2014}.
Theory suggests that when comparing policies with asymmetric support under the $\Delta$-OPE framework (e.g. a deterministic production model vs. a high-entropy exploration model), the optimal design is non-uniform.

\textit{Setup.}
We compare a conservative control policy $\pi'$ (allocating 95\% mass to a single ``safe'' action) against an exploratory treatment policy $\pi$ (allocating 50\% mass uniformly across 100 actions). We vary the traffic allocation $p \in [0.05, 0.95]$ and measure the empirical variance of the $\Delta$-IPS estimator over $N=5\,000,000$ trials.

\textit{Result.}
Figure~\ref{fig:optimal_design} validates our optimal design theory. The empirical variance perfectly tracks the theoretical variance $J(p)$ derived in Eq.~\ref{eq:variance_treatment_alloc} (up to a constant factor).
The minimum variance is achieved at $p^\star \approx 0.81$, allocating significantly more traffic to the exploratory treatment policy.
This optimal split yields a variance reduction of approximately 18\% compared to the standard 50/50 split.
This confirms that experimental design should proportionally favour higher-entropy policies to ensure adequate coverage of the action space.
As $J(p)$ can be estimated from pre-experiment data, this confirms that $p^{\star}$ can be estimated effectively and efficiently.

\begin{figure}[!t]
    \centering
    \includegraphics[width=\linewidth]{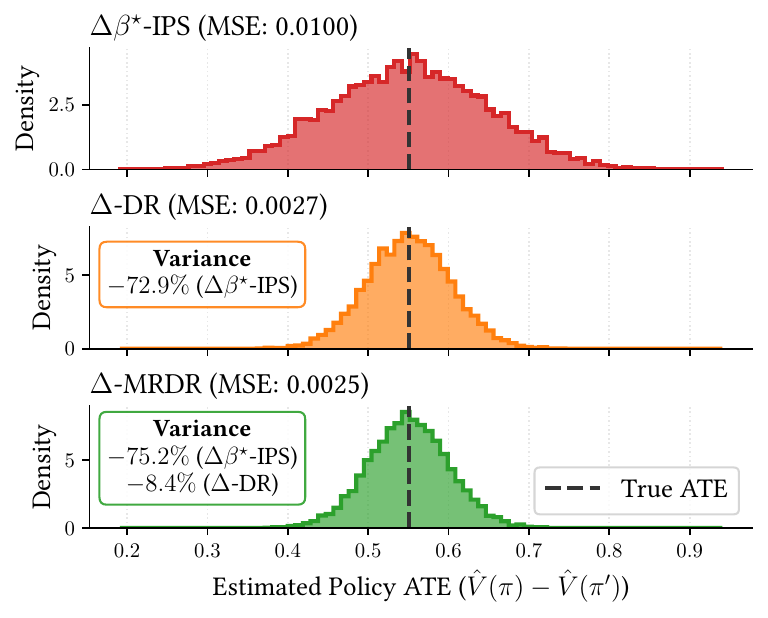}
    \caption{Variance reduction via the $\Delta$-MRDR objective in a highly heterogeneous environment. The $\Delta\beta^\star$-IPS estimator serves as the optimal model-free baseline. The $\Delta$-DR estimator minimises global MSE, causing it to overfit to the noisy majority and mischaracterise the baseline for the target group, resulting in suboptimal variance. The $\Delta$-MRDR estimator targets model capacity exclusively at the regions of high policy divergence, learning the correct local relationship and significantly reducing ATE estimation variance.}
        \label{fig:mrdr_results}
\end{figure}
\subsection{RQ3: Variance Reduction via $\Delta$-MRDR}
We evaluate whether weighting the regression loss by the policy divergence reduces the variance of the resulting Doubly Robust (DR) estimator, particularly when model capacity is constrained.

\textit{Setup.}
We simulate a highly heterogeneous recommendation environment where a latent intent feature $q \in (0, 1)$ governs both user behaviour and observation noise.
The majority of the traffic consists of low-intent users whose outcomes are highly stochastic.
A small minority consists of high-intent users whose outcomes have lower variance.
This heteroskedasticity provides an environment in which we expect uniformly weighted DR to be suboptimal.

We compare a target policy $\pi$ that aggressively targets high-intent users against a passive control baseline $\pi'$ using data logged from a standard $50/50$ A/B-test.
We train Ridge regression models with moderate regularization to force a strict bias-variance trade-off, utilising a two-stage sample splitting procedure to ensure unbiased evaluation:
\begin{enumerate}
    \item \textbf{$\Delta$-DR Model:} Minimises global MSE,
    \item \textbf{$\Delta$-MRDR Model:} Minimises the weighted MSE with sample weights $w_i^2 = \left(\frac{\pi(a_i|x_i) - \pi'(a_i|x_i)}{\pi_0(a_i|x_i)}\right)^2$.
\end{enumerate}

\textit{Result.}
Figure~\ref{fig:mrdr_results} compares the distribution of ATE estimates over $10\,000$ trials.
First, we see that the $\Delta$-DR estimator significantly reduces the variance of the $\Delta\beta^{\star}$-IPS estimator whilst retaining unbiasedness.
Nevertheless, with the regularisation term constraining model capacity, the reward model focuses on the majority of users, who are low-intent and noisy.
In contrast, the $\Delta$-MRDR model focuses on users with a large causal impact for the ATE in policy change.
This targeted learning approach reduces the estimator variance by $75\%$ and $8\%$ compared to $\Delta\beta^{\star}$ and $\Delta$-DR respectively.
These results demonstrate that the best model for prediction is not necessarily the best model for causal estimation.

\begin{figure}[!t]
    \centering
    \includegraphics[width=\linewidth]{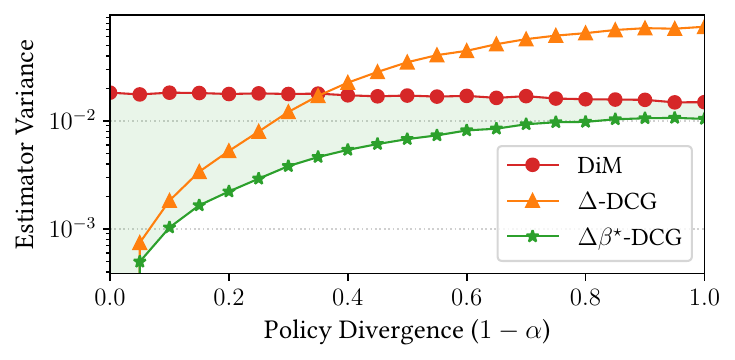}
    \caption{Empirical variance of top-$K$ ranking estimators across varying levels of policy divergence ($1-\alpha$), shown on a logarithmic scale. The standard A/B testing procedure (DiM) suffers from high variance regardless of policy overlap. The unbaselined $\Delta$-DCG estimator  performs well under high overlap but suffers from severe variance explosion as divergence increases. The $\Delta\beta^\star$-DCG estimator utilises position-specific optimal baselines to mitigate this explosion, achieving significant variance reduction compared to the standard A/B test across all divergence regimes.}
    \label{fig:dcg_variance}
\end{figure}
\subsection{RQ4: Variance Reduction in Ranking}\label{sec:exp_delta_dcg}

Finally, we evaluate the performance of our exposure-aware estimators in combinatorial action spaces, specifically focusing on top-$K$ ranking under a Position-Based Model (PBM). We investigate how policy divergence impacts estimation variance and whether the position-specific optimal baseline ($\Delta\beta^\star_{\indep}$-DCG) can successfully mitigate the variance explosion inherent to off-policy ranking evaluation---translating to direct statistical power improvements for online evaluation of ranking policies.

\textit{Setup.}
We simulate a ranking environment with a catalogue of $M=100$ items and a slate size of $K=25$ (i.e. $|\mathcal{A}| \approx 2.4\times10^{23}$).
User contexts and item relevances $q$ are randomly distributed.
We define a target policy $\mathcal{R}$ that perfectly sorts items by $q$, and a control policy $\mathcal{R}'$ whose scores are corrupted by noise. The degree of noise is controlled by an interpolation parameter $\alpha \in [0, 1]$, allowing us to strictly sweep the policy divergence ($1 - \alpha$) from $0.0$ (identical policies) to $1.0$ (independent policies with limited overlap). 

Data is logged using a $50/50$ uniform mixture of $\mathcal{R}$ and $\mathcal{R}'$, with a small $\epsilon=0.1$ uniform exploration mass added to ensure strict positivity of exposure propensities. We compare three estimators over $5\,000$ independent trials: the standard Difference-in-Means (DiM) of the observed A/B-test DCG, the unbaselined $\Delta$-DCG estimator, and the optimally baselined $\Delta\beta_{\indep}^\star$-DCG estimator which computes a distinct $\beta_{\indep,j}^\star$ for each rank position $j$.

\textit{Result.}
Figure~\ref{fig:dcg_variance} illustrates the empirical variance of the estimators (on a logarithmic scale) as policy divergence increases.
The standard A/B-testing (DiM) variance remains relatively constant across all divergence levels, as it simply computes the sample mean difference between the two logging arms without exploiting the structural overlap of the policies.

When policies are highly similar (divergence $\to 0$), both $\Delta$-estimators leverage the substantial overlap in item exposures to achieve near-zero variance, vastly outperforming DiM. However, as policy divergence increases, the exposure propensities for the two policies diverge.
For the unbaselined $\Delta$-DCG estimator, this causes the marginal importance weights to inflate, leading to a variance explosion that surpasses the standard DiM baseline. 

Crucially, the $\Delta\beta_{\indep}^\star$-DCG estimator mitigates this explosion. By shifting the outcome at each position $j$ by the variance-minimising scalar $\beta^\star_{\indep,j}$, the estimator successfully neutralises the variance injected by the importance weights while preserving the benefits of the $\Delta$-OPE formulation. The resulting $\Delta\beta_{\indep}^\star$-DCG estimator maintains a strict variance advantage over the standard A/B-test across the entire divergence spectrum, demonstrating the efficacy of position-aware control variates when evaluating ranking policies.

\subsection{Limitations of our Experimental Design}
The empirical validation presented in this work relies exclusively on synthetic simulation environments.
Whilst this choice affords full control over ground-truth counterfactuals---enabling exact bias and variance measurement against the true ATE, which is unobservable in logged public data---it necessarily entails simplifying assumptions about reward structure, context distributions, and policy behaviour. 
Practitioners should treat the reported variance reduction figures as indicative of the relative ordering of methods, and empirical validation of our theoretical contributions, rather than as precise predictions of real-world uplifts.
Indeed, these will depend on the degree of policy overlap, reward heterogeneity, and the quality of any fitted reward models in their specific deployment.

The absence of publicly available logged datasets suitable for this experimental setting is a structural limitation of the broader field, not unique to this work.
Existing benchmarks for off-policy evaluation in recommendation~\cite{Swaminathan2015_BLBF} tend to involve small action spaces and rely on counterfactual evaluation procedures that are themselves subject to high variance~\cite{Lefortier2016}---making them ill-suited as a reliable ground truth against which to validate estimator efficiency.
In this respect, our simulation approach follows established practice in the OPE literature~\cite{Saito2021_OBP}, where synthetic environments are the standard vehicle for controlled comparison of estimators.

A second limitation concerns the estimation of policy probabilities. The proposed estimators (\textsc{$\Delta$-IPS}, \textsc{$\Delta$-DR}, \textsc{$\Delta$-MRDR}, \textsc{$\Delta$-DCG}) require knowledge of $\pi(a|x)$ and $\pi'(a|x)$ to compute the importance weights $w_\Delta$.
In our simulations, and often in practice, these probabilities are known exactly by construction.
Alternatively, policies can come from stochastic neural models whose exact action probabilities may be expensive to compute~\cite{Jeunen2025_TS} or subject to numerical instability---particularly in large action spaces or in sequence-generation settings such as large language models, where the policy probability over a full output is a product of per-token probabilities.
We note that this is a general challenge for IPS-based methods, and that existing mitigations---such as top-$K$ truncation~\cite{Ionides2008, Roux2025} or embedding-based approximations~\cite{Saito2022_MIPS}---are likely applicable here. Nevertheless, the sensitivity of the proposed estimators to propensity misspecification warrants further empirical investigation.

Finally, the derivation of the variance-optimal traffic allocation $p^\star$ assumes stationarity of the context distribution $\mathsf{P}(X)$ between the pre-experiment logging period and the live experiment, and invokes homoskedasticity of outcomes $Y$ for the analytical solution under $\Delta$-IPS. The adaptive two-stage design we suggest relaxes the latter assumption in practice, but the robustness of $p^\star$ to distribution shift---common in live recommender systems due to seasonal effects, item catalogue changes, or user population drift---remains an open question. We view this as a promising direction for future work, alongside the development of fully sequential adaptive designs that converge to $p^\star$ continuously over the course of the experiment whilst retaining unbiased estimators of ATEs.

%% file: 5.Conclusions.tex
\section{Conclusions \& Outlook}
\label{sec:conclusion}

Our work identifies a structural inefficiency in the standard A/B-testing protocol: when competing policies agree on an action, the resulting outcome contributes noise to the Average Treatment Effect estimate without providing any signal regarding the uplift.
By reframing the randomised treatment assignment as a meta-policy, we established a rigorous connection between online controlled experiments and Off-Policy Evaluation.
This perspective unlocks a family of unbiased, policy-aware estimators ($\Delta$-IPS and $\Delta$-DR) that strictly dominate standard Difference-in-Means (DiM) estimators by exploiting policy overlap.

Our theoretical analysis demonstrates that the variance of these estimators scales with policy divergence rather than direct outcome noise.
This insight enables three practical innovations:
\begin{enumerate*}[label=(\roman*)]
    \item a derivation of the variance-optimal traffic allocation ratio $p^\star$, which can deviate from the standard 50/50 split;
    \item the $\Delta$-MRDR learning objective, which trains control variates to explicitly minimise estimation variance by focusing model capacity on regions of policy disagreement; and
    \item the extension of this framework to combinatorial ranking action spaces via $\Delta$-DCG, leveraging a Position-Based Model and rank-specific optimal baselines ($\Delta\beta_{\indep}^\star$-DCG) to successfully combat variance explosion in Top-$K$ ranking evaluation.
\end{enumerate*}
Empirical validation confirms that these methods significantly reduce estimation variance in both highly heterogeneous single-action environments and complex slate-recommendation scenarios, maintaining strict dominance over standard DiM across varying divergence regimes with minimal engineering overhead.

This work opens several avenues for future research.
First, while our $\Delta\beta_{\indep}^\star$-DCG estimator relies on a marginal approximation for computational stability, future work could explore estimating the full cross-rank covariance matrix to construct a strictly globally optimal joint baseline vector for slate evaluation~\cite{Jeunen2026_AdditiveDominates}.
Second, our optimal design derivation assumes a fixed allocation $p$.
Further exploring an adaptive experimental design, where traffic allocation is dynamically updated to converge to $p^\star$ during the test, could drastically accelerate decision-making.
Third, whilst our work focuses on ATEs, recent extensions of the OPE framework allow us to estimate a full distribution of counterfactual outcomes~\cite{Chandak2021}---which could be integrated directly.
Finally, we see promise in applying $\Delta$-MRDR beyond variance reduction, potentially using it as a robust loss function for offline policy learning in settings with skewed logging policies.
We hope this work encourages the community to view online and offline evaluation not as distinct phases, but as a unified continuum of counterfactual estimation problems.\looseness=-1